\documentclass[letterpaper, 10 pt, conference]{ieeeconf}  

\IEEEoverridecommandlockouts                              
\usepackage{graphics} 
\usepackage{epsfig} 
\usepackage{mathptmx} 
\usepackage{times} 
\usepackage{amsmath} 
\usepackage{amssymb}  
\usepackage{booktabs}
\usepackage{algorithm}
\usepackage{algpseudocode}
\usepackage{url} 
\usepackage{caption}

\usepackage{amsthm}

\theoremstyle{plain}      
\newtheorem{theorem}{Theorem}
\newtheorem{lemma}{Lemma}

\newtheorem{corollary}{Corollary}
\newtheorem{assumption}{Assumption}

\theoremstyle{definition} 
\newtheorem{definition}{Definition}
\newtheorem{remark}{Remark}

\usepackage{todonotes}

\newcommand{\VaR}{\mathrm{VaR}}
\newcommand{\CVaR}{\mathrm{CVaR}}

\title{\LARGE \bf
Federated Distributional Reinforcement Learning \\ with Distributional Critic Regularization
}

\author{David Millard, Cecilia Alm, Rashid Ali, Pengcheng Shi and Ali Baheri
\thanks{D. Millard and A. Baheri are with the Dept. of Mechanical Engineering,
        Rochester Institute of Technology, Rochester, NY 14623, USA
        {\tt\small djm3622@rit.edu,akbeme@rit.edu}}%
\thanks{C. Alm is with the Dept. of Psychology and School of Information,
        Rochester Institute of Technology, Rochester, NY 14623, USA
        {\tt\small coagla@rit.edu}}%
\thanks{R. Ali is with the Department of Engineering Science, University West, Trollhättan 461 30, Sweden,
        {\tt\small rashid.ali@hv.se}}%
\thanks{P. Shi is with the Golisano College of Computing and Information Sciences,
        Rochester Institute of Technology, Rochester, NY 14623, USA
        {\tt\small spcast@rit.edu}}%
}
\begin{document}

\maketitle
\thispagestyle{empty}
\pagestyle{empty}

\begin{abstract}
Federated reinforcement learning typically aggregates value functions or policies by parameter averaging, which emphasizes expected return and can obscure statistical multimodality and tail behavior that matter in safety-critical settings. We formalize federated distributional reinforcement learning (FedDistRL), where clients parametrize quantile value function critics and federate these networks only. We also propose TR-FedDistRL, which builds a per client, risk-aware Wasserstein barycenter over a temporal buffer. This local barycenter provides a reference region to constrain the parameter averaged critic, ensuring necessary distributional information is not averaged out during the federation process. The distributional trust region is implemented as a shrink–squash step around this reference. Under fixed-policy evaluation, the feasibility map is nonexpansive and the update is contractive in a probe-set Wasserstein metric under evaluation. Experiments on a bandit, multi-agent gridworld, and continuous highway environment show reduced mean-smearing, improved safety proxies (catastrophe/accident rate), and lower critic/policy drift versus mean-oriented and non-federated baselines.
\end{abstract}


\section{INTRODUCTION}
As deep reinforcement learning (RL) controllers are deployed in an increasing range of systems, the need for sample-efficient learning from real-world interaction grows accordingly. In many such application domains (e.g., autonomous driving and personal robotics), data is privacy-sensitive and cannot leave the local device, motivating federated reinforcement learning (FRL), where clients learn from local experience while sharing only value-network/policy parameters for federated aggregation.

However, comparatively little work in FRL directly targets tail-risk objectives, instead emphasizing expected-return optimization or constraint-based safety formulations \cite{koursioumpas2024safe,lu2024risk}. Moreover, under environmental heterogeneity, federated parameter averaging can induce output-space smoothing in learned value estimates, effectively collapsing distinct return distributions toward an ``average'' behavior that may be mismatched and potentially unsafe for individual clients \cite{peng2025neural}. Motivated by these issues, we study an actor-critic federated distributional RL (FedDistRL) setting in which the policy network (responsible for mapping states to actions) is kept local, while the critic (a value function estimator used to evaluate the policy and perform policy improvement) uses a risk-augmented advantage surrogate. While this enables risk-aware value sharing (via critic parameter sharing), aggregation-induced smoothing can still attenuate the multimodality and tail structure that motivates distributional value estimates. To mitigate this effect, we propose a distribution-space trust-region mechanism: each client constructs a local, risk-aware reference via a CVaR-weighted Wasserstein barycenter over a buffer containing recent critic outputs and constrains the next-round critic update to remain close to this reference. This constraint prevents aggregation-induced parameter updates from collapsing locally observed multimodal or heavy-tailed return distributions. Our contributions are (i) we identify mean-smearing in parameter aggregation and demonstrate it empirically, (ii) we introduce FedDistRL, federating distributional critics  while keeping policies local, (iii) we propose TR-FedDistRL, a barycentric regularization method that biases critic updates toward safety via a CVaR-weighted reference distribution, and (iv) we provide a theoretical stability result for the constrained critic update and demonstrate empirical improvements in safety-related metrics.


\section{RELATED WORKS}

\subsection{Federated Reinforcement Learning} 
Federated learning classically aggregates model updates via iterative averaging to learn from decentralized data under communication constraints \cite{mcmahan2017communication}.  Extending this paradigm to reinforcement learning, federated reinforcement learning is commonly categorized into horizontal and vertical settings depending on whether clients share state--action representations \cite{qi2021federated}.  A central technical obstacle is environment heterogeneity (distinct transition kernels and reward processes), for which recent work analyzes aggregation-induced optimality gaps and proposes averaging variants such as QAvg/PAvg \cite{jin2022federated}. Related FRL directions include personalization via distance constraints \cite{xiong2024personalized} and offline/ensemble-driven experience sharing \cite{rengarajan2024federated}. However, the dominant practice in FRL remains mean-oriented: value/policy aggregation is typically designed around expected return, which can obscure multimodality and tail behavior that are critical in safety-sensitive deployments.

\subsection{Risk-Averse Reinforcement Learning} 
Distributional reinforcement learning models the full return law rather than only its expectation \cite{bellemare2017distributional}, with practical function-approximation methods such as quantile regression and implicit quantile networks \cite{dabney2018distributional,dabney2018implicit}.  Risk-sensitive objectives based on lower-tail criteria (e.g., CVaR) have a foundational optimization characterization \cite{rockafellar2000optimization} and have been integrated into RL via CVaR-MDP formulations and policy-gradient style methods \cite{NIPS201564223ccf,tamar2015optimizing}.  Recent work connects CVaR optimization explicitly with distributional RL and highlights subtleties of risk-based action selection under distributional Bellman operators \cite{NEURIPS2022c88a2bd0}.  In safety-critical settings, distributional critics have also been used to enforce risk-aware constraints (e.g., distributional safety critics and trust-region style safe distributional actor-critic methods) \cite{yang2023safety,kim2023trust}.  Separately, optimal transport provides geometry-aware aggregation via Wasserstein barycenters \cite{villani2008optimal,agueh2011barycenters}, which have been used in RL for uncertainty propagation and exploration \cite{metelli2019propagating,likmeta2023wasserstein} and have recently appeared in heterogeneous federated RL (as an alternative parameter averaging mechanism) \cite{pereira2025heterogeneous}. Our work uses a risk-weighted Wasserstein barycenter to define a local trust region that regularizes the federated critic toward the previous round's measure of risk.

\subsection{Trust Regions \& Reinforcement Learning}
Trust-region methods stabilize policy-gradient learning by constraining each update to remain close to the current policy, classically via a KL-divergence bound (TRPO) \cite{schulman2015trust}. PPO approximates this principle with clipped (or penalized) surrogates that provide much of TRPO's empirical stability with simpler optimization \cite{schulman2017proximal}. For safety-constrained learning, CPO extends TRPO-style updates to enforce expected-cost constraints during training \cite{achiam2017constrained}. In multi-agent settings, trust-region guaranties are more delicate due to non-stationarity induced by other agents; recent works derive MARL trust-region variants with sequential updates and monotonic-improvement style results \cite{kuba2021trust}. Furthermore, work has also explored replacing KL with Wasserstein-1 trust regions for multi-agent coordination, achieving success in escaping local optima \cite{salgarkar2025distance}.

\subsection{Optimal Transport \& Reinforcement Learning}
OT has been increasingly used in RL as a geometric tool for comparing and regularizing distributions beyond mean-based objectives \cite{baheri2024synergy}. In risk-sensitive RL, OT has been used to formulate risk-aware objectives through distances between visitation distributions and desired risk profiles \cite{baheri2023risk}. OT has also been applied in inverse reinforcement learning to study reward ambiguity through Wasserstein geometry \cite{baheri2023understanding}. More recently, Wasserstein-based regularization has been incorporated into actor-critic methods to stabilize value estimation \cite{baheri2025wave}, and barycenter-based coordination mechanisms have been proposed for cooperative multi-agent RL \cite{baheri2025wasserstein}. 

\begin{figure*}[ht]
    \centering
    {
        \includegraphics[width=0.32\textwidth]{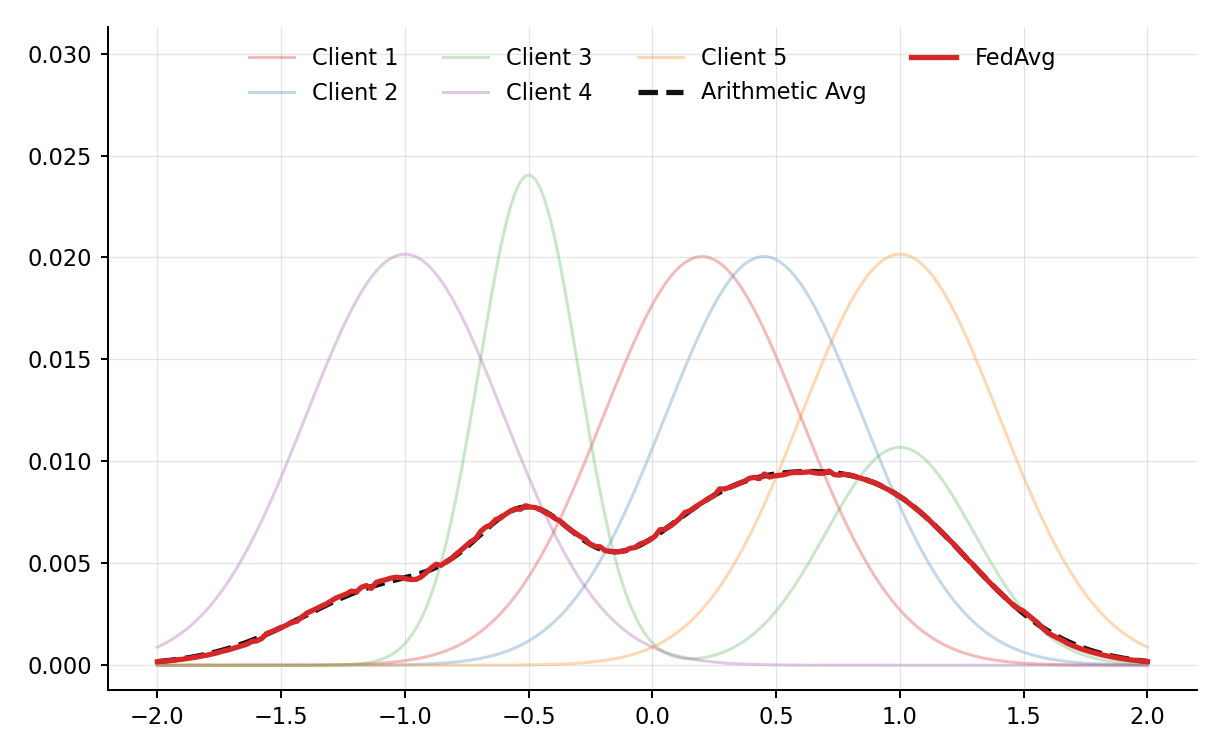}
    }\hfill
    {
        \includegraphics[width=0.32\textwidth]{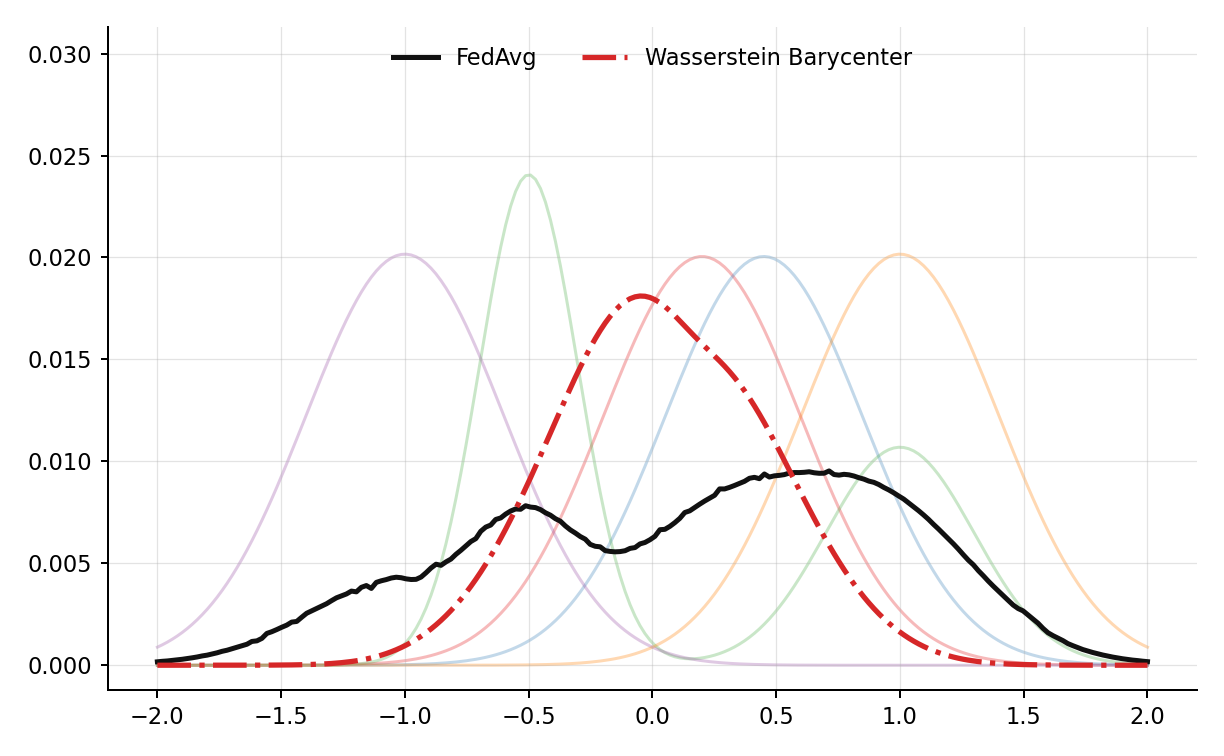}
    }\hfill
   {
        \includegraphics[width=0.32\textwidth]{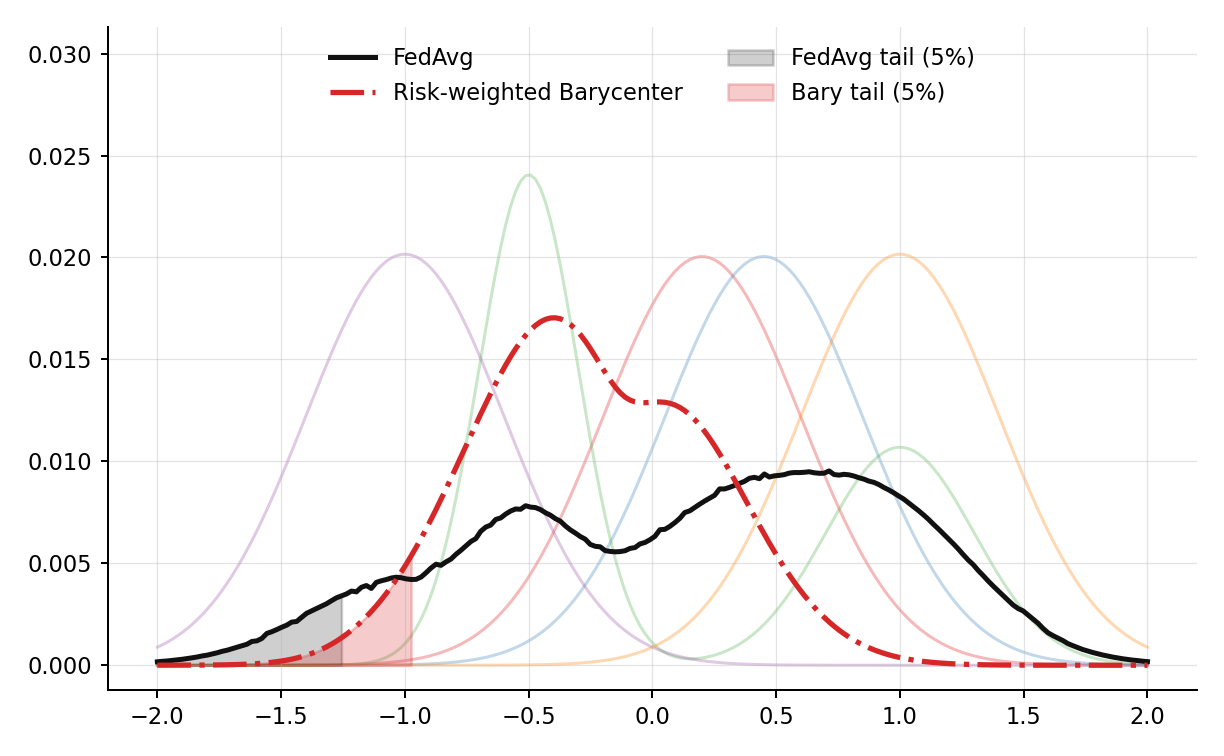}
    }
    \caption{Each panel shows how aggregation operators combine scalar, multimodal client reward distributions in the bandit setting (i.e., $Z^\pi=r_t$, $\gamma=0$). Each client is parameterized with QR-DQN models and averaging is computed via Equation~\ref{eq:param_avg}.
\emph{Left:} critic parameter averaging (\textsc{FedAvg}) induces an output-space arithmetic averaging effect that can smear modes and attenuate tail mass (mean-smearing).
\emph{Middle:} an unweighted Wasserstein-1 barycenter aggregates in distribution space and better preserves geometric structure (modes/tails) compared with arithmetic averaging.
\emph{Right:} a CVaR-risk-weighted Wasserstein-1 barycenter biases the barycenter toward lower-tail behavior by upweighting clients with larger lower-tail risk (CVaR at level $\alpha = 0.1$), helping retain tail-relevant mass while still preserving multimodality.}
    \label{fig:bandit}
\end{figure*}


\section{PRELIMINARIES \& PROBLEM FORMULATION}

\subsection{Distributional Reinforcement Learning}

Distributional reinforcement learning models the return as a random variable rather than its expectation.  
For a policy $\pi$, the return distribution is defined as:
\begin{align}
Z^\pi(s,a)
&\triangleq
\sum_{t=0}^{\infty} \gamma^t r_t(s_t,a_t),
\label{eq:return_def}
\end{align}
where $a_t \sim \pi(\cdot \mid s_t)$ and $s_{t+1} \sim P(\cdot \mid s_t,a_t)$.
The corresponding state--action value function is recovered as
$
Q^\pi(s,a)
=
\mathbb{E}\!\left[ Z^\pi(s,a) \right].
$
Distributional methods retain the full law $Z^\pi(s,a) \in \mathcal{P}(\mathbb{R})$.
The distributional Bellman evaluation operator associated with a fixed policy $\pi$ is
\begin{align}
(T^\pi Z)(s,a)
\overset{d}{=}
R(s,a)
+
\gamma Z(S',A'),
\end{align}
where $A' \sim \pi(\cdot\mid S'),\;
S' \sim P(\cdot\mid s,a)$. For $p \ge 1$, this operator is a $\gamma$-contraction in the $p$-Wasserstein metric and admits a unique fixed point corresponding to $Z^\pi$
\cite{bellemare2017distributional}.
The distributional control operator is defined by
\begin{align}
(T Z)(s,a)
\overset{d}{=}
R(s,a)
+
\gamma Z\!\left(
S',
\arg\max_{a'} \mathbb{E}[Z(S',a')]
\right).
\end{align}
Unlike its expectation-based counterpart, this operator is not a contraction in Wasserstein distance and may fail to admit a unique fixed point
\cite{bellemare2017distributional}. 
\begin{definition}[Empirical-measure lift and induced metric]
\label{def:empirical_metric}
For any $x\in\mathbb R^K$, define the associated empirical measure on $\mathbb R$ by
$\mu_x \triangleq \frac{1}{K}\sum_{k=1}^K \delta_{x_k}.
$For $p\ge 1$, define the empirical Wasserstein metric on $\mathbb R^K$ by
$d_K(x,y) \triangleq W_p(\mu_x,\mu_y).
$ Given a quantile critic $q:S\times A\to\mathbb R^K$, define the induced metric on critics by
$d(q,q') \triangleq \sup_{(s,a)} d_K\!\big(q(s,a),q'(s,a)\big).
$\end{definition}

\subsection{Wasserstein Barycenter}
Optimal Transport (OT) provides a geometric framework for comparing probability
distributions based on mass displacement \cite{villani2008optimal,agueh2011barycenters}. A central object
is the Wasserstein barycenter, defined for distributions $\{\mu_i\}_{i=1}^N$ and weights
$\{\beta_i\}$ as
\begin{equation}
\mu^\star
=
\arg\min_{\mu}
\sum_{i=1}^N
\beta_i\, W_p^p(\mu,\mu_i).
\end{equation}
Unlike Euclidean averaging, Wasserstein barycenters preserve geometric
structure of distributions, including modes and tails.
We consider a CVaR-weighted variant.
Given a random variable $X\sim\mu$, the Conditional Value-at-Risk (CVaR) at level
$\alpha\in(0,1]$ is defined as
\begin{equation}
\CVaR_\alpha(X)
=
\frac{1}{\alpha}
\int_{0}^{\alpha}
\VaR_u(X)\,du,
\end{equation}
where $\VaR_u(X)=\inf\{x:\mathbb{P}(X\le x)\ge u\}$.
Using CVaR-induced weights, we define a risk-weighted Wasserstein barycenter as
$
\mu^\star_{\mathrm{risk}}
=
\arg\min_{\mu\in\mathcal{P}(\mathbb{R})}
\sum_{i=1}^N
\beta_i\, W_p^p(\mu,\mu_i),
\label{eq:risk_weighted_barycenter}
$
where the weights $\beta_i$ are computed from a lower-tail CVaR statistic of each
distribution.
Since $\CVaR_\alpha$ uses this lower tail, distributions with larger tail behavior
($\mathcal{R}_i$) are up-weighted, biasing the barycenter to preserve risk in the event parameter averaging washes it away.

\subsection{Problem Formulation}

We consider $N$ clients $\{F_i\}_{i=1}^N$, each interacting exclusively with a local environment $E_i$. HFRL assumes an aligned state--action structure with heterogeneous dynamics; for all $i \neq j$,
\begin{align}
S_i = S_j = S, \quad A_i = A_j = A, \quad E_i \neq E_j, \quad E_i \in \mathcal{G},
\end{align}
where $S$ and $A$ are shared state and action spaces, and $\mathcal{G}$ is a family of environments. Each $E_i$ is modeled as
$
M_i = \langle S, A, P_i, R_i, \gamma \rangle,
$
with differing transition kernel $P_i$, reward distribution $R_i$, and discount $\gamma \in [0,1)$.
Each client maintains a policy $\pi_i(\cdot|s)$. The distributional return is
$
Z_i^{\pi_i}(s,a)
=
\sum_{t=0}^{\infty} \gamma^t R_i(s_t,a_t),
$
with $a_t \sim \pi_i(\cdot|s_t)$ and $s_{t+1} \sim P_i(\cdot|s_t,a_t)$. Its evolution satisfies
\begin{align}
(T^{\pi_i}Z)(s,a)
\overset{d}{=}
R_i(s,a) + \gamma Z(S',A').
\end{align} 
\begin{definition}[FedDistRL with risk-aware reference]
\label{def:feddistrl_reference}
Now consider a critic-only federated distributional reinforcement learning setting with
$N$ heterogeneous clients. Each client $F_i$ interacts with its local MDP $M_i$ and maintains
(i) a local policy $\pi_i\in\Pi$ and (ii) a parameterized distributional critic
$\hat Z_{\theta_i}:\mathcal S\times\mathcal A\to \mathcal P(\mathbb R)$, with
\begin{align}
\mathcal{Z}_\Theta \triangleq \{\hat Z_\theta:\mathcal S\times\mathcal A\to \mathcal P(\mathbb R)\mid \theta\in\Theta\}.
\end{align}
At round $t$, client $F_i$ collects private experience $\mathcal D_i^t$ and performs local
actor-critic updates, optionally using a risk-augmented advantage surrogate.
At the end of each round, only critic parameters are communicated. The server aggregates critic parameters via
a parameter average (as implemented in standard FL) \cite{mcmahan2017communication},
\begin{align}\label{eq:param_avg}
\theta_{\mathrm{srv}}^{Z}(t)
=
\mathcal{M}\big(\theta_1^{Z}(t),\dots,\theta_N^{Z}(t); w_1,\dots,w_N\big),
\end{align}
$\mathcal{M}:\Theta^N\to\Theta,$
inducing the broadcast critic $\hat Z_{\theta_{\mathrm{srv}}^{Z}(t)}$.
Each client initializes its next-round critic from the broadcast parameters and then optionally applies
a client-local regularization step relative to a risk-aware reference.
Formally, client $F_i$ constructs a risk-aware reference return law
$\bar Z_i^t\in\mathcal P(\mathbb R)^{\mathcal S\times\mathcal A}$ from a temporal buffer
$\{\hat Z_{\theta_i^{Z}(m)}\}_{m=t-B+1}^t$ of its recent critic outputs via a map
\begin{align}
\bar Z_i^t
=
\mathcal{R}\big(\hat Z_{\theta_i^{Z}(t-B+1)},\dots,\hat Z_{\theta_i^{Z}(t)}\big),
\end{align}
$\mathcal{R}:\mathcal{Z}_\Theta^B\to \mathcal P(\mathbb R)^{\mathcal S\times\mathcal A},
$where $\mathcal{R}$ emphasizes downside behavior.
The reference $\bar Z_i^t$ is not communicated to the server; it is used only to define
a client-local trust region that constrains the post-aggregation critic update.
\end{definition}
\begin{remark}
In the formulation outlined in Defintion~\ref{def:feddistrl_reference}, risk only sensitivity enters through (i) the local policy-improvement surrogate
and (ii) the construction of the client-local reference $\bar Z_i^t$, while the server-side
aggregation $\mathcal{M}$ remains mean-oriented. Furthermore, the actor $\pi_\phi$ parameters remain untouched, leading to a smaller chance of federation to destabilize the decision making process.
\end{remark}

\begin{figure*}[ht]
    \centering
    \includegraphics[width=0.8\linewidth]{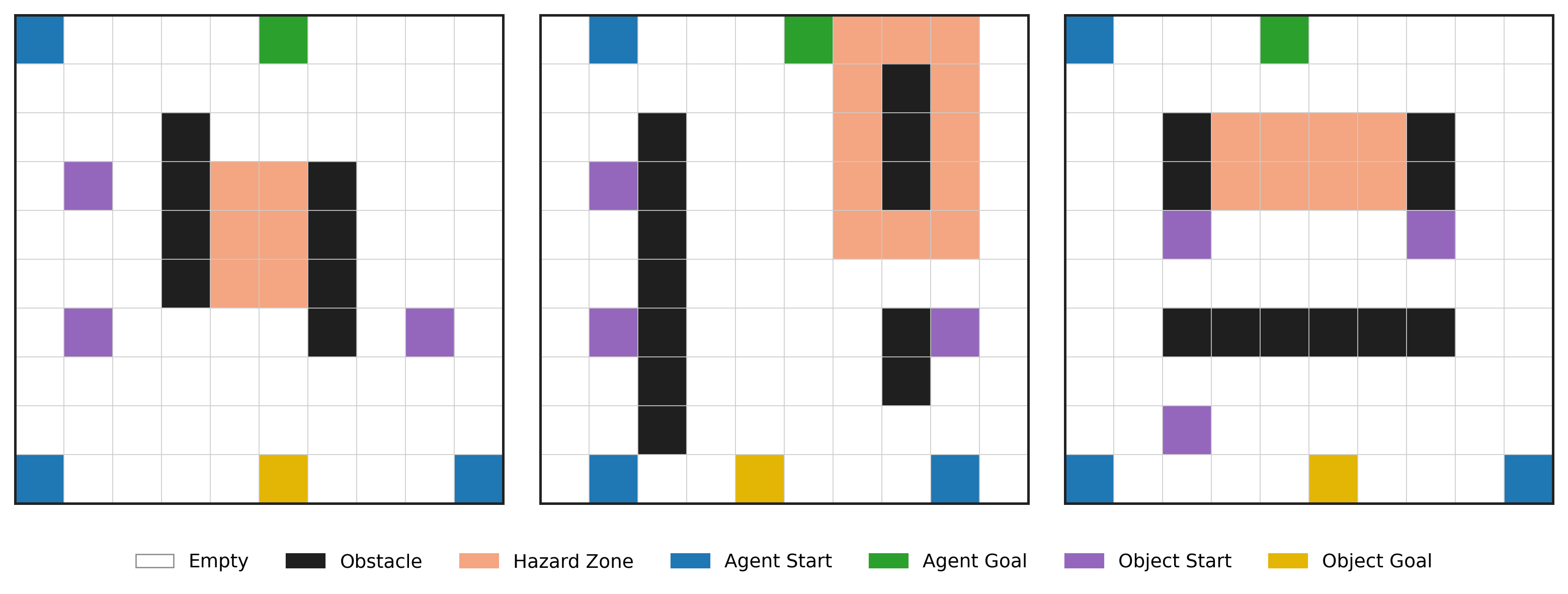}
    \caption{Case Study 2 client environments for Clients 1--3. Each panel shows one client’s environment with obstacles, hazard zone, agent start/goal locations, and object start/goal locations. The three client environments are arranged for direct side-by-side comparison of environmental heterogeneity. All instantiations use a $10\times10$ grid with 3 agents and 3 objects.}
    \label{fig:layouts_2}
\end{figure*}


\section{METHODOLOGY}

\subsection{Local Updates}

This framework works with policy-gradient methods such as PPO and MAPPO, where the critic influences the policy update exclusively through the advantage estimator. Federating only the distributional critic enables the exchange of risk-sensitive value information across clients without directly influencing the policy representations.

Each client optimizes a local clipped policy-gradient objective. For PPO, the actor objective is
\begin{equation}
\mathcal{L}_{\text{actor}}
=
\mathbb{E}_t\!\left[
\min\!\left(
r_t(\theta)\hat{A}^{\text{comb}}_t,\,
\mathrm{clip}\!\left(r_t(\theta),1-\epsilon,1+\epsilon\right)\hat{A}^{\text{comb}}_t
\right)
\right].
\end{equation}
where $r_t(\theta)=\pi_\theta(a_t|s_t)/\pi_{\theta_{\text{old}}}(a_t|s_t)$ is the importance ratio. 
We define a downside-risk-sensitive state--action value by applying lower-tail CVaR
to the return law:
\begin{align}
Q_{\text{mean}}^\pi(s,a) &\triangleq \mathbb{E}\!\left[ Z^\pi(s,a) \right], \\
Q_{\text{cvar}}^\pi(s,a) &\triangleq \CVaR_{\alpha_{\text{cvar}}}\!\left( Z^\pi(s,a) \right).
\end{align}
where CVaR uses the lower tail $(0,\alpha]$ with small $\alpha$. Next, we then optimize a surrogate risk-sensitive objective
\begin{equation}
J_{\lambda}(\pi)
\triangleq
\mathbb{E}_{s_0\sim\rho_0}\!\left[
V_{\text{mean}}^\pi(s_0)
+
\frac{\lambda_{\text{cvar}}}{\tau_{\text{cvar}}}\,V_{\text{cvar}}^\pi(s_0)
\right],
\label{eq:risk_objective}
\end{equation}
$V_{\bullet}^\pi(s)\triangleq \mathbb{E}_{a\sim\pi(\cdot|s)}\!\left[Q_{\bullet}^\pi(s,a)\right]$. If risk-sensitivity is not used, this collapses to the standard policy gradient. With a quantile critic $\hat Z_\theta(s,a)=\{z_{\tau_k}(s,a)\}_{k=1}^K$,
we estimate CVaR by discretizing the tail integral:
\begin{equation}
\widehat{\CVaR}_{\alpha_{\text{cvar}}}\!\big(\hat Z_\theta(s,a)\big)
\triangleq
\sum_{\ell=1}^{L}\alpha_\ell\, z_{u_\ell}(s,a),
\label{eq:cvar_discrete_quantiles}
\end{equation}
$u_\ell\in(0,\alpha_{\text{cvar}}],\ \alpha_\ell\ge 0,\ \sum_{\ell=1}^L \alpha_\ell=1,$ where uniform weights $\alpha_\ell=1/L$ recover the standard CVaR discretization.
In particular, given a
$K$-quantile critic with fixed quantile fractions $\tau_k\triangleq (k-\tfrac12)/K$, we choose
    $u_\ell \in \{\tau_k:\ \tau_k\le \alpha_{\text{cvar}}\}$,
$L \triangleq \max\{1,\lfloor \alpha_{\text{cvar}}K\rfloor\}$,
and set $z_{u_\ell}(s,a)$ to the critic output at level $u_\ell$, i.e.,
$z_{u_\ell}(s,a)=F^{-1}_{\hat Z_\theta(s,a)}(u_\ell)$ as represented by the corresponding predicted
quantile. Uniform weights $\alpha_\ell=1/L$ recover the standard Riemann-sum approximation
$\widehat{\CVaR}_{\alpha_{\text{cvar}}}(\hat Z_\theta(s,a)) \approx \frac{1}{\alpha_{\text{cvar}}}\int_0^{\alpha_{\text{cvar}}}
F^{-1}_{\hat Z_\theta(s,a)}(u)\,du$. Monotonicity is also enforced by the critic parameterization.
The combined advantage estimator is then defined as:
\begin{equation}
\hat{A}^{\text{comb}}_t
=
\hat{A}^{\text{mean}}_t
+
\lambda_{\text{cvar}}
\frac{\hat{A}^{\text{cvar}}_t}{\tau_{\text{cvar}}},
\end{equation}
where $\hat{A}^{\text{mean}}_t$ is a standard advantage estimator using
$Q_{\text{mean}}$ (e.g., GAE computed from $\mathbb{E}[\hat Z]$), and
$\hat{A}^{\text{cvar}}_t$ is an advantage estimator computed analogously using
$Q_{\text{cvar}}$ via \eqref{eq:cvar_discrete_quantiles}.
The critic is trained using distributional quantile regression with a vector-valued
Bellman target
$
Z_{\text{tgt}}
=
r_t + \gamma(1-d_t)\,Z_{\text{next}},
$
where $Z_{\text{next}}$ denotes the target quantiles at the next state.
The critic loss is the pairwise quantile Huber loss \cite{nair2025significant}:
\begin{equation}
\mathcal{L}_{\text{critic}}
=
\frac{1}{K^2}
\sum_{i=1}^{K}\sum_{j=1}^{K}
\rho_{\tau_i}^{\kappa}
\!\left(
Z_{\text{tgt},j} - Z_i(s_t)
\right),
\end{equation}
At the end of a round, only the parameters of $\hat{Z}$ are exchanged during federated aggregation.

\subsection{Aggregation}
At the end of each federated round, all participating clients $F_i \in \mathcal{F}_n$ transmit their current distributional critic parameters to the central server. In this work, aggregation is performed via parameter averaging of the critic networks. Let $\theta_i^Z$ denote the parameters of the distributional critic at client $F_i$. The aggregated critic is:
\begin{equation}
\theta_{\text{avg}}^Z
=
\sum_{i \in \mathcal{F}_n} w_i \, \theta_i^Z,
\end{equation}
where the weights $w_i \ge 0$ satisfy $\sum_{i \in \mathcal{F}_n} w_i = 1$. The weights reflect client-specific factors such as compute budget or number of local update steps performed during the round. This aggregation assumes that all participating critics share an identical architecture, which is standard in federated learning. If this assumption does not hold, server-side distillation \cite{seo202216} can be used instead, under a fixed probe set of states to align heterogeneous critics. Following aggregation, all participating clients initialize their critic parameters to $\theta_{\text{avg}}^Z$ before the start of the next round.

\begin{remark}
Under environmental heterogeneity, the post-aggregation critic may not reflect the downside (tail) behavior relevant to a particular client. In particular, parameter aggregation can smooth or distort tail structure, yielding a critic whose lower-quantile predictions are misaligned with a client's reality. To mitigate this aggregation-induced risk mismatch, we apply a client-local trust-region regularization step that anchors the next-round critic update to a risk-aware reference constructed from the client's own recent critic outputs, thereby limiting drift in worst-case (tail) estimates across rounds.
\end{remark}

\begin{table*}[t]
\caption{Results are mean $\pm$ standard deviation over 30 random seeds. \textsc{Return} is the average episodic return (higher is better). \textsc{Output Drift} quantifies across-round instability of the distributional critic outputs on a fixed evaluation/probe set (lower means the critic’s predicted return laws change less from round to round and are less perturbed by federation). \textsc{Catastrophe} is the evaluation catastrophe rate (fraction of episodes ending in a catastrophic event; lower is better).}
\label{tab:cdc_fedrl_results_std}
\centering
\small
\begin{tabular}{lccc}
\toprule
Method & Return $\uparrow$ & Output Drift $\downarrow$ & Catastrophe $\downarrow$ \\
\midrule
Local
& 483.09 $\pm$ 128.40
& \textit{113.37} $\pm$ 27.94
& 0.0323 $\pm$ 0.0076 \\
FedAvg
& \textbf{554.76} $\pm$ 95.51
& 159.71 $\pm$ 18.56
& 0.0240 $\pm$ 0.0054 \\
FedAvg (CVaR)
& \textit{544.56} $\pm$ 74.24
& 151.51 $\pm$ 17.09
& \textit{0.0203} $\pm$ 0.0037 \\
FedAvg (CVaR + TR)
& 436.21 $\pm$ 54.13
& \textbf{98.07} $\pm$ 11.77
& \textbf{0.0172} $\pm$ 0.0032 \\
\bottomrule
\end{tabular}
\end{table*}

\subsection{Trust Region}

\subsubsection{Initialization}
In addition to standard local updates, each client maintains a buffer of its own
distributional critic outputs across rounds.
For client $F_i$, let $\boldsymbol{\beta}_i^{(n)}(s,a)=\{\beta_{i,m}(s,a)\}_{m=1}^n$
denote CVaR-based weights over the buffered laws.
From this buffer, client $F_i$ computes a local Wasserstein-$1$ barycenter
\begin{equation}
Z^{\mathrm{bc},(n)}_{i}(s,a)
\in
\arg\min_{Z\in\mathcal{P}(\mathbb{R})}
\sum_{m=1}^{n} \beta_{i,m}(s,a)\,
W_1\!\Big(Z,\, Z_{i}^{(m)}(s,a)\Big),
\label{eq:w1_barycenter_buffer}
\end{equation}
with $\beta_{i,m}(s,a)\ge 0$ and $\sum_{m=1}^n \beta_{i,m}(s,a)=1$.
Given quantile outputs $\{\hat q_{i,m,k}(s,a)\}_{k=1}^K$ at levels
$\tau_k=(k-\tfrac12)/K$, this lower-tail estimate at level $\alpha\in(0,1]$ is
$\widehat{\mathrm{CVaR}}_{\alpha}\!\left(Z_i^{(m)}(s,a)\right)
=
\frac{1}{K_\alpha}\sum_{k:\,\tau_k\le\alpha}\hat q_{i,m,k}(s,a)
$, 
$
K_\alpha=\{k:\tau_k\le\alpha\}.
$ Next, define a tail-risk score
$r_{i,m}(s,a)=-\widehat{\mathrm{CVaR}}_{\alpha}\!\left(Z_i^{(m)}(s,a)\right),
$ and convert scores to barycenter weights via
\begin{equation}
\beta_{i,m}(s,a)=
\frac{\exp\!\big(\lambda\, r_{i,m}(s,a)\big)}
{\sum_{\ell=1}^{n}\exp\!\big(\lambda\, r_{i,\ell}(s,a)\big)},
\qquad \lambda>0.
\label{eq:beta_from_cvar}
\end{equation}
Because returns are scalar, $W_1$ admits a quantile function representation:
for any $Z$ with quantile function $F_Z^{-1}$,
$W_1(Z,Z')
=
\int_0^1 \big|F_Z^{-1}(u)-F_{Z'}^{-1}(u)\big|\,du.
$ Thus, for each $u\in(0,1)$, the Wasserstein-$1$ barycenter quantile is a
$\beta_{i,m}(s,a)$-weighted median:
\begin{equation}
F^{-1}_{Z^{\mathrm{bc},(n)}_{i}(s,a)}(u)
\in
\arg\min_{q\in\mathbb{R}}
\sum_{m=1}^{n}\beta_{i,m}(s,a)\,\big|q - F^{-1}_{Z_i^{(m)}(s,a)}(u)\big|.
\label{eq:w1_barycenter_quantile_median}
\end{equation}
This barycenter is computed locally and acts as the reference distribution. This reference serves as a client-side trust-region that constrains federated updates.

\begin{figure}
    \centering
    \includegraphics[width=1.0\linewidth]{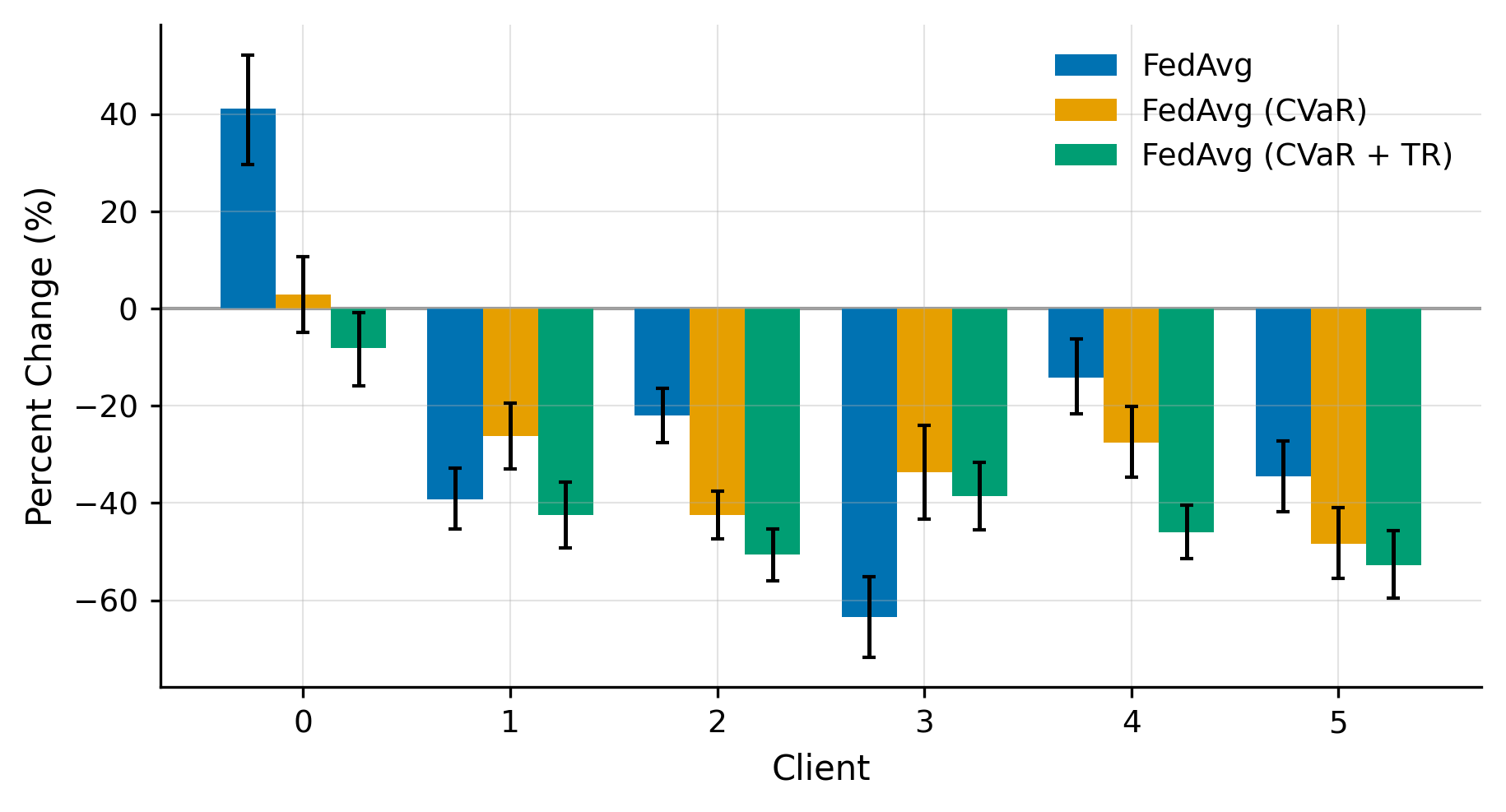}
    \caption{Bars report the percent change in catastrophe rate for each heterogeneous client relative to the \textsc{Local} baseline (baseline corresponds to 0\% by definition), aggregated over 30 random seeds; negative values indicate fewer catastrophes (safer) than \textsc{Local}. Catastrophe rate is the fraction of evaluation episodes that terminate in a designated catastrophic event (environment-defined unsafe terminal condition). Error bars denote variability across seeds (standard deviation).}
    \label{fig:safety_case_2}
\end{figure}

\subsubsection{Proximal Update}
During round $n{+}1$, each client constrains its local critic toward the reference
$\bar Z_i^{(n)}$ computed at the end of round $n$. Let $T_i$ denote the unconstrained
distributional Bellman update.
For each quantile $k$, define an asymmetric tube around $\bar Z_{i,k}^{(n)}$ with radii
$\delta_k^{-}$ (downside) and $\delta_k^{+}$ (upside), and the operator
\begin{align}
\Phi_k(x;\bar Z_{i,k}^{(n)})
=
\bar Z_{i,k}^{(n)}+
\begin{cases}
\delta_k^{+}\tanh\!\left(\beta_k\dfrac{x-\bar Z_{i,k}^{(n)}}{\delta_k^{+}}\right), & x\ge \bar Z_{i,k}^{(n)},\\[6pt]
\delta_k^{-}\tanh\!\left(\beta_k\dfrac{x-\bar Z_{i,k}^{(n)}}{\delta_k^{-}}\right), & x< \bar Z_{i,k}^{(n)}.
\end{cases}
\end{align}
The constrained critic update is implemented via a shrink and squash step:
\begin{align}
\tilde Z_{i,k}
&=
\bar Z_{i,k}^{(n)}
+
\alpha\Big(
T_{i,k}(Z_i)
-
\bar Z_{i,k}^{(n)}
\Big),
\qquad \alpha\in(0,1],\\
Z_{i,k}^{+}
&=
\Phi_k\!\big(\tilde Z_{i,k};\bar Z_{i,k}^{(n)}\big).
\end{align}

\begin{lemma}[Tube squash: anchoring and Lipschitz in $d$]
\label{lem:phi_lip_d}
For each round $n$, assume the tube squash $\Phi^{(n)}(\cdot;\bar q)$ acts coordinatewise and satisfies:
(i) anchoring $\Phi^{(n)}(\bar q;\bar q)=\bar q$; and
(ii) coordinatewise Lipschitz for all $x,y\in\mathbb R^K$ and all $k$,
$
|(\Phi^{(n)}(x;\bar q))_k-(\Phi^{(n)}(y;\bar q))_k|
\le \beta_k^{(n)}|x_k-y_k|.
$
Let $\beta_{\max}^{(n)}\triangleq \max_k \beta_k^{(n)}$. Then for all $x,y\in\mathbb R^K$,
$
W_p\!\big(\Psi_K(\Phi^{(n)}(x;\bar q)),\,\Psi_K(\Phi^{(n)}(y;\bar q))\big)
\le
\beta_{\max}^{(n)}\,
W_p\!\big(\Psi_K(x),\,\Psi_K(y)\big).
$
Consequently, under Assumption~\ref{assump:eval_metric}, for any critics $q,q'$,
$
d\!\big(\Phi^{(n)}(q;\bar q),\Phi^{(n)}(q';\bar q)\big)
\le
\beta_{\max}^{(n)}\, d(q,q').
$
\end{lemma}

\begin{lemma}[Tube-cap bound in $d$]
\label{lem:tube_cap_d}
Let $\delta_k^{(n)}\triangleq \max\{\delta_k^{-,(n)},\delta_k^{+,(n)}\}$ and define
$
\Delta_{\mathrm{tube}}^{(n)} \triangleq
\left(\frac{1}{K}\sum_{k=1}^K (\delta_k^{(n)})^p\right)^{1/p}.
$
By construction of the tube squash, $|(\Phi^{(n)}(x;\bar q))_k-\bar q_k|\le \delta_k^{(n)}$
for all $k$ and all $x\in\mathbb R^K$. Then for any $\bar q\in\mathbb R^K$,
$
W_p\!\big(\Psi_K(\Phi^{(n)}(x;\bar q)),\,\Psi_K(\bar q)\big)\le \Delta_{\mathrm{tube}}^{(n)}.
$
Equivalently, under Assumption~\ref{assump:eval_metric}, for any critic $q$ and reference $\bar q$,
$
d\!\big(\Phi^{(n)}(q;\bar q),\bar q\big)\le \Delta_{\mathrm{tube}}^{(n)}.
$
\end{lemma}


\section{THEORETICAL RESULTS}
Next we provide a stability guarantee for the critic-side trust-region step used in our method.
The result shows that the trust-region update is a perturbed contraction toward a (moving) reference and that the tube parameters enforce an absolute deviation cap from that reference.
We first fix a client $F_i$. Let $q_i^{(n)}(s,a)\in\mathbb{R}^K$ denote the quantile critic at
round $n$, and let $\bar q_i^{(n)}(s,a)$ denote a (possibly data-dependent) reference
at the same round. Define the tracking error to the moving reference
\begin{equation}
e_n \triangleq d\!\big(q_i^{(n)},\bar q_i^{(n)}\big)
= \max_{(s,a)\in\mathcal P} W_1\!\big(\Psi_K(q_i^{(n)})(s,a),\,\Psi_K(\bar q_i^{(n)})(s,a)\big),
\end{equation}
where $\Psi_K(q)(s,a)\triangleq \frac{1}{K}\sum_{k=1}^K \delta_{q_k(s,a)}$ and $\mathcal P$ is a fixed finite probe set.
Since our critic parameterization enforces monotone quantiles by construction, the approximate backup
$\widehat T_{i,K}^{\pi}$ returns a valid quantile vector.
Define the shrink step
$
\tilde q_{i}^{(n+1)}
\triangleq (1-\alpha)\,\bar q_i^{(n)} + \alpha \big(\widehat T_{i,K}^{\pi_{n+1}} q_i^{(n)}\big)$, 
$\alpha\in(0,1]
$,
and the ideal trust-region map
$
q_{i,\mathrm{ideal}}^{(n+1)}
\triangleq \Phi^{(n)}(\tilde q_i^{(n+1)};\bar q_i^{(n)}).
$
The realized iterate can deviate from the ideal map due to finite SGD, target-network lag,
or implementation details. We model this by a per-round discrepancy $\eta_n\ge 0$:
\begin{equation}
d\!\big(q_i^{(n+1)},q_{i,\mathrm{ideal}}^{(n+1)}\big) \le \eta_n.
\label{eq:eta_def}
\end{equation}
We also define (i) a reference residual, (ii) a backup approximation/sampling error,
and (iii) a reference drift:
\begin{align}
\varepsilon_n
&\triangleq
d\!\big(\bar q_i^{(n)},\,T_{i,K}^{\pi_{n+1}}\bar q_i^{(n)}\big),
\label{eq:eps_def}\\
\xi_n
&\triangleq
d\!\big(\widehat T_{i,K}^{\pi_{n+1}} q_i^{(n)},\,T_{i,K}^{\pi_{n+1}} q_i^{(n)}\big),
\label{eq:xi_def}\\
r_n
&\triangleq
d\!\big(\bar q_i^{(n+1)},\bar q_i^{(n)}\big).
\label{eq:rn_def}
\end{align}

\begin{assumption}[Evaluated critic metric]
\label{assump:eval_metric}
Assume either $(\mathcal S,\mathcal A)$ is finite, or there exists a fixed finite probe set
$\mathcal P\subseteq \mathcal S\times\mathcal A$ such that every occurrence of $d(\cdot,\cdot)$ is interpreted as
\begin{align}
d(q,q') \triangleq \max_{(s,a)\in\mathcal P} W_1\!\big(\Psi_K(q)(s,a),\Psi_K(q')(s,a)\big),
\end{align}
where $\Psi_K(q)(s,a)\triangleq \frac{1}{K}\sum_{k=1}^K \delta_{q_k(s,a)}$. 
\end{assumption}

\begin{figure*}[t]
    \centering
    \includegraphics[width=1.0\linewidth]{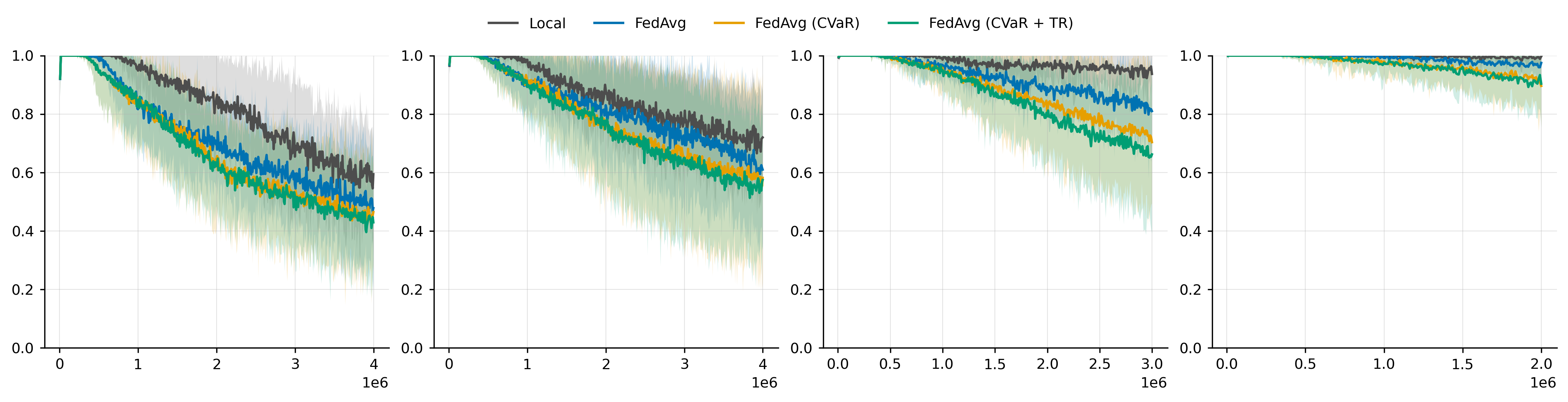}
    \caption{Accident rate across training steps for each heterogeneous client in the highway environment (Case Study 3), comparing \textsc{Local}, \textsc{FedAvg}, \textsc{FedAvg (CVaR)}, and \textsc{FedAvg (CVaR + TR)} over 30 random seeds. Shaded regions denote standard deviation across seeds. Lower accident rate indicates safer behavior. \textsc{FedAvg (CVaR + TR)} consistently achieves the lowest accident rate across clients, demonstrating that the proposed risk-aware trust-region regularization improves safety in a continuous-control setting under environment heterogeneity.}
    \label{fig:full_rates_3}
\end{figure*}

\begin{assumption}[Evaluation contraction in $d$]\label{assump:ideal_contr_d}
For each client $i$ and round $n$, the projected evaluation operator
$
T^{\pi_{n+1}}_{i,K}
\;\triangleq\;
\Pi_{W_1}\circ T^{\pi_{n+1}}_i
$
is a $\gamma$-contraction in the probe-set metric
\begin{align}
d(q,q')
\;\triangleq\;
\max_{(s,a)\in\mathcal P}
W_1\!\big(\Psi_K(q)(s,a),\,\Psi_K(q')(s,a)\big).
\end{align}
This is justified by:
(a) $\gamma$-contraction of $T^{\pi}$ in the maximal $W_1$ metric for policy
evaluation, and
(b) nonexpansiveness of the quantile projection $\Pi_{W_1}$ in $W_1$,
as established in \cite{dabney2018distributional}.
\end{assumption}

\begin{theorem}[Stability around a moving reference]
\label{thm:anchor_tracking_d}
Assume Assumption~\ref{assump:eval_metric}
and Assumption~\ref{assump:ideal_contr_d} holds with parameters
$\gamma\in[0,1)$.
Let $\beta_{\max}^{(n)}$ be the Lipschitz constant from Lemma~\ref{lem:phi_lip_d}.
Then the tracking error $e_n$ satisfies
\begin{equation}
e_{n+1}
\le
\alpha\,\beta_{\max}^{(n)}\big(\gamma e_n + \varepsilon_n + \xi_n \big)
\;+\; r_n \;+\; \eta_n.
\label{eq:main_recursion_d}
\end{equation}
Moreover, the tube squash induces a cap around the previous reference
(Lemma~\ref{lem:tube_cap_d}), yielding
\begin{equation}
e_{n+1}
\le
\min\!\Big\{
\alpha\,\beta_{\max}^{(n)}\big(\gamma e_n + \varepsilon_n + \xi_n \big),
\ \Delta_{\mathrm{tube}}^{(n)}
\Big\}
\;+\; r_n \;+\; \eta_n.
\label{eq:main_recursion_min_d}
\end{equation}
\end{theorem}
\begin{proof}[Proof]
By triangle inequality in the metric $d$,
\begin{align}
e_{n+1}
=
d\!\big(q_i^{(n+1)},\bar q_i^{(n+1)}\big)
\le
& \;d\!\big(q_i^{(n+1)},q_{i,\mathrm{ideal}}^{(n+1)}\big)\\
&+
d\!\big(q_{i,\mathrm{ideal}}^{(n+1)},\bar q_i^{(n)}\big)\\
&+
d\!\big(\bar q_i^{(n)},\bar q_i^{(n+1)}\big).
\end{align}
The first and third terms are bounded by $\eta_n$ and $r_n$ via \eqref{eq:eta_def} and
\eqref{eq:rn_def}.
For the middle term, apply Lemma~\ref{lem:phi_lip_d} (anchoring and
$\beta_{\max}^{(n)}$-Lipschitz in $d$ for fixed $\bar q_i^{(n)}$):
\begin{align}
d\!\big(q_{i,\mathrm{ideal}}^{(n+1)},\bar q_i^{(n)}\big)
&=
d\!\Big(\Phi^{(n)}(\tilde q_i^{(n+1)};\bar q_i^{(n)}),\bar q_i^{(n)}\Big)\\
&\le
\beta_{\max}^{(n)}\, d\!\big(\tilde q_i^{(n+1)},\bar q_i^{(n)}\big).
\end{align}
Then for each $(s,a)$,
$\tilde q(s,a) = (1-\alpha)\bar q(s,a) + \alpha \widehat T q(s,a)$.
Letting $\mu_{\tilde q} \triangleq \Psi_K(\tilde q)$ and $\mu_{\bar q}\triangleq \Psi_K(\bar q)$, we use the
one-dimensional monotone coupling. Under the standard quantile ordering condition, the vectors
$\tilde q(s,a)$, $\bar q(s,a)$, and $(\widehat T q)(s,a)$ are nondecreasing in $k$ and represent atoms on a common
quantile grid with equal weights. Hence the optimal transport pairs corresponding indices and
$
W_1\!\big(\mu_{u},\mu_{v}\big)=\frac{1}{K}\sum_{k=1}^K |u_k-v_k|.
$
Applying this with $u=\tilde q(s,a)$ and $v=\bar q(s,a)$ and using
$\tilde q_k(s,a)-\bar q_k(s,a)=\alpha\big((\widehat T q)_k(s,a)-\bar q_k(s,a)\big)$ yields
\begin{align}
W_1\!\big(\mu_{\tilde q(s,a)},\mu_{\bar q(s,a)}\big)
=
\alpha\,W_1\!\big(\mu_{(\widehat T q)(s,a)},\mu_{\bar q(s,a)}\big),
\end{align}
and hence, taking the maximum over $(s,a)\in\mathcal P$,
\begin{equation}
d\!\big(\tilde q_i^{(n+1)},\bar q_i^{(n)}\big)
\le
\alpha\, d\!\big(\widehat T_{i,K}^{\pi_{n+1}} q_i^{(n)},\bar q_i^{(n)}\big).
\end{equation}
Now add and subtract $T_{i,K}^{\pi_{n+1}} q_i^{(n)}$ and $T_{i,K}^{\pi_{n+1}}\bar q_i^{(n)}$,
use triangle inequality, Assumption~\ref{assump:ideal_contr_d}, and
\eqref{eq:eps_def}--\eqref{eq:xi_def}:
\begin{align}
d\!\big(\widehat T_{i,K}^{\pi_{n+1}}  q_i^{(n)},\bar q_i^{(n)}\big)
&\le
d\!\big(\widehat T_{i,K}^{\pi_{n+1}} q_i^{(n)},T_{i,K}^{\pi_{n+1}} q_i^{(n)}\big)\\
&\quad +
d\!\big(T_{i,K}^{\pi_{n+1}} q_i^{(n)},T_{i,K}^{\pi_{n+1}} \bar q_i^{(n)}\big)\\
&\quad +
d\!\big(T_{i,K}^{\pi_{n+1}} \bar q_i^{(n)},\bar q_i^{(n)}\big)\\
&\le
\xi_n + \gamma e_n + \varepsilon_n.
\end{align}
Substituting the bounds yields \eqref{eq:main_recursion_d}.
The capped form \eqref{eq:main_recursion_min_d} follows by Lemma~\ref{lem:tube_cap_d}.
\end{proof}

\begin{table}[t]
\caption{Results are mean $\pm$ standard deviation over 30 random seeds. \textsc{Return} is the average episodic return (higher is better). \textsc{Accident} is the evaluation accident rate (fraction of episodes with at least one collision/accident event; lower is better). Methods mirror Case Study 2.}
\label{tab:case_3}
\centering
\small
\begin{tabular}{lccc}
\toprule
Method & Return $\uparrow$ & Accident $\downarrow$ \\
\midrule
Local
& 23.197 $\pm$ 4.830
& 0.811 $\pm$ 0.063 \\
FedAvg
& 27.797 $\pm$ 4.219
& 0.736 $\pm$ 0.068 \\
FedAvg (CVaR)
& 30.168 $\pm$ 3.996
& 0.678 $\pm$ 0.066 \\
FedAvg (CVaR + TR)
& \textbf{31.231} $\pm$ 3.943
& \textbf{0.653} $\pm$ 0.066 \\
\bottomrule
\end{tabular}
\end{table}

\begin{corollary}[Geometric tracking bound]
\label{cor:geometric_bound_d}
Suppose $\beta_{\max}^{(n)}\le \bar\beta$ for all $n$ and define
$
\rho \triangleq \alpha \bar\beta \gamma$, $
d_n \triangleq \alpha \bar\beta(\varepsilon_n+\xi_n) + r_n + \eta_n.
$
If $\rho<1$ and $d_n\le \bar d$ uniformly, then
$e_n \le \rho^n e_0 + \frac{1-\rho^n}{1-\rho}\,\bar d
\le \rho^n e_0 + \frac{\bar d}{1-\rho}.
$\end{corollary}
\begin{remark}
Theorem~\ref{thm:anchor_tracking_d} is a critic-side stability statement
$d(q,q')=\max_{(s,a)\in\mathcal P} W_1(\Psi_K(q)(s,a),\Psi_K(q')(s,a))$.
Algorithmically, it implies (i) if the defect/perturbation terms are controlled and $\Phi^{(n)}$ remains Lipschitz (bounded $\beta_{\max}^{(n)}$),
then the client critic stays close to its moving reference across rounds (Corollary~\ref{cor:geometric_bound_d}),
preventing large jumps in evaluated return-law predictions after server aggregation, and
(ii) when $\bar q_i^{(n)}$ is constructed to emphasize downside behavior (e.g., a CVaR-weighted barycentric reference over
recent rounds), stability in $d$ limits how quickly aggregation-induced smoothing can erode tail-relevant structure,
up to the tube radius $\Delta_{\mathrm{tube}}^{(n)}$ and drift $r_n$.
In our numerical results, the reported drift-based metrics serve as a proxy for this reference-drift effect, and the reduced drift under the trust-region variant is consistent with the predicted stability around a moving anchor.
\end{remark}


\section{NUMERICAL RESULTS \& DISCUSSION}

\subsection{Case Study 1: Bandits}
We first motivate our mean-smearing claims with Case Study~1. This case study isolates the effect by using a bandit experiment with aggregation in return-law space and parameter space. As shown in Fig.~\ref{fig:bandit}, FedAvg collapses to an arithmetic averaging effect on the induced empirical return measures, introducing multi-modality unseen in the individual client distributions and weakening tail mass. In contrast, an unweighted Wasserstein barycenter better preserves the underlying geometric structure (modality), while the CVaR-weighted barycenter further inflates the aggregate toward larger estimates of downside behavior while retaining the prior structure. These qualitative behaviors directly motivate our choice to construct a client-local, risk-aware reference from recent critic outputs and to regularize updates via a distributional trust region anchored at that reference.

\subsection{Case Study 2: Mutliagent Gridworld}

Next, we evaluate a heterogeneous six-client multi-agent sparse-reward gridworld with small probability hazard zones. Three clients are visualized in Figure~\ref{fig:layouts_2}. We report the metric mean $\pm$ standard deviation over 30 seeds in Table~\ref{tab:cdc_fedrl_results_std} and client-wise safety changes in Fig.~\ref{fig:safety_case_2}. Each client environment is a heterogeneous $10\times 10$ warehouse with $3$ agents and $3$ movable objects, static obstacles, a hazard zone with one active hazard cell sampled per episode, a shared agent goal, and an object goal. Rewards are sparse and event-driven,
$
r_t \triangleq R_{\mathrm{obj}}\mathbf{1}\{\Delta N_{\mathrm{obj},t}>0\}+R_{\mathrm{goal}}\mathbf{1}\{\mathcal{G}_t\},
$
where $\mathcal{G}_t$ denotes the event that all objects have been delivered and all agents are at their goal at time $t$. A catastrophe occurs when any agent enters the active hazard cell and a Bernoulli trigger fires with probability $p_{\text{haz}}$. We report the accident rate as $\mathbb{E}\!\left[\mathbf{1}\{\mathrm{collision}\}\right]$ and output drift on a fixed probe set $\mathcal P$ of randomly-sampled states as $D_{\text{out}}^{(r)}=\frac{1}{|\mathcal P|}\sum_{s\in\mathcal P}\|f_{\phi^{(r)}}(s)-f_{\phi^{(r-1)}}(s)\|_1$. Heterogeneity is introduced client-wise by assigning clients different obstacle layouts, start/goal/object/hazard-zone placements, hazard probability, slip probability, episode horizon, step cost, and catastrophe penalty.
 Relative to \textsc{Local}, federated aggregation improves mean return by increasing sample efficiency through the federation mechanism; however, federation also increases output drift, measured as the temporal distance between a randomly sampled probe set gathered early in training. Incorporating CVaR reduces the catastrophe rate, and adding the trust-region step (\textsc{FedAvg (CVaR + TR)}) yields both the lowest drift and the lowest catastrophe rate in Table~\ref{tab:cdc_fedrl_results_std}. Figure~\ref{fig:safety_case_2} further indicates that these safety gains persist across heterogeneous clients. Overall, the results support the claim that anchoring critic updates to a client-local, risk-aware reference stabilizes return-law estimates across rounds and improves safety metrics under heterogeneity while revealing a negative impact on return under sparse rewards. Furthermore, the result of lower return indicates policy over optimization. Given the small probability that entering a hazard zone actually leads to termination, dangerous policies are incentivized to take the risk for more return.

\subsection{Case Study 3: Highway Environment}

We evaluate a five-client continuous-control highway driving task \cite{highway-env}, with safety measured by the accident rate. Heterogeneity is introduced with differing vehicle counts, vehicle density, target speeds, and reward-speed ranges across clients. We report the average discounted return $J(\pi)\triangleq\mathbb{E}\!\left[\sum_{t\ge 0}\gamma^t r_t\right]$, the accident rate $\mathbb{E}\!\left[\mathbf{1}\{\mathrm{collision}\}\right]$, and a policy-drift proxy given by PPO's minibatch approximate KL divergence $\widehat{\mathrm{KL}}(\pi_{\theta_{n+1}}\|\pi_{\theta_n})$.
Given that this continuous-state setting does not admit the environment-resampling protocol used in Case Study~2, we use $\widehat{\mathrm{KL}}$ as a coarse drift proxy and still observe consistent reductions under \textsc{FedAvg (CVaR + TR)}. Table~\ref{tab:case_3} and Figure~\ref{fig:safety_case_2} report mean $\pm$ standard deviation over 30 seeds. Consistent with Case Study~2, \textsc{FedAvg (CVaR + TR)} dominates the baselines across all evaluation metrics, achieving the highest return while simultaneously reducing policy drift and accident rate. These results indicate that the proposed risk-aware, trust-region-anchored aggregation remains effective in a continuous-control setting, where safety violations correspond to rare but high-impact events, and demonstrate improved robustness without sacrificing performance. Intuitively, this result is expected, as return and safety are more tightly coupled than in Case Study~2. In Case Study~2, safety corresponds to a fixed hazardous region within the gridworld that the agent must learn to avoid. In contrast, in this setting, vehicles are dynamically moving, which more strongly incentivizes safe behavior and leads to a closer alignment between return and safety.


\section{CONCLUSION}
We introduced federated distributional reinforcement learning (FedDistRL), where clients federate distributional critics while keeping policies local, enabling risk-relevant value sharing under privacy constraints. We identified mean-smearing as a characteristic failure mode of mean-oriented aggregation for distributional critics, which can degrade multimodality and the tail structure that matter for safety. To mitigate this, we proposed TR-FedDistRL, which constructs a client-local risk-aware reference using a CVaR-weighted Wasserstein barycenter over a round buffer and constrains critic updates with a distributional trust region. Across heterogeneous benchmarks, the proposed regularization reduced drift and improved safety metrics relative to mean-oriented baselines while maintaining competitive return. Future work includes extending theoretical guaranties beyond evaluation to control, improving reference construction under nonstationarity, and integrating heterogeneous critic architectures via distillation or distribution-space aggregation.

\subsection{Limitations}
This work has several limitations. Our approach builds on parameter averaging, which is an early and widely used federated learning paradigm that is known to be sensitive to client heterogeneity; while the proposed risk-aware trust-region regularization mitigates aggregation-induced drift, it does not fully address the structural limitations of parameter-space aggregation. In addition, the baselines considered are necessarily limited and do not include more recent federated optimization or personalization methods that may further improve robustness under heterogeneity. On the theory side, our analysis is restricted to the stability properties of the critic-side update and does not provide global convergence guaranties for the full actor-critic dynamics. Finally, the empirical evaluation is limited to a small number of environments and safety metrics, and extending the framework to richer safety constraints or larger-scale deployments remains an open direction.


\section{ACKNOWLEDGMENTS}

This material is based upon work supported by the National Science Foundation under 
Award No. OISE-2420109. Any opinions, findings, conclusions, or recommendations 
expressed in this material are those of the author(s) and do not necessarily reflect the views of the National Science Foundation.


\section{CODE AVAILABILITY}
Our repository is available under the Apache License version 2.0 at \url{https://github.com/djm3622/fedrl}. Configuration files provide hyperparameters utilized for final results.


\bibliographystyle{unsrt}
\bibliography{ref}

\end{document}